\documentclass[12pt]{colt2018} 

\title[Neural network hypothesis space]{The effect of the choice of neural network depth and breadth on the size of its hypothesis space}
\usepackage{times}
\usepackage{appendix}
\usepackage{multirow}
\usepackage{enumitem}  
\usepackage{thmtools, thm-restate}

  \coltauthor{\Name{Lech Szymanski} \Email{lechszym@cs.otago.ac.nz}\\
   \Name{Brendan McCane} \Email{mccane@cs.otago.ac.nz}\\
   \Name{Michael Albert} \Email{malbert@cs.otago.ac.nz}\\
   \addr University of Otago}

\newcommand{\M}{n}
\newcommand{\U}{U}

\renewcommand{\L}{L}
\renewcommand{\l}{l}
\renewcommand{\P}{W}
\newcommand{\W}{V}
\newcommand{\V}{\mathcal{V}}
\renewcommand{\S}{\W^{\P}}
\newcommand{\Uf}{\U_{\l}!}
\newcommand{\Ufprod}{\prod_{\l}\Uf}
\newcommand{\Hbound}{\S/\Ufprod}
\renewcommand{\H}{\mathcal{H}}
\newcommand{\Hsize}[1]{|\H_{#1}|}
\newcommand{\Y}{\mathcal{Y}}
\newcommand{\Ysize}[1]{|\Y_{#1}|}
\newcommand{\f}{\sigma}

\usepackage{ifthen}
\newboolean{draft}
\setboolean{draft}{false}

\begin{document}

\maketitle

\begin{abstract}
We show that the number of unique function mappings in a neural network hypothesis space is inversely proportional to $\Ufprod$, where $U_{\l}$ is the number of neurons in the hidden layer $\l$.  
\end{abstract}


\section{Introduction}

A shallow neural network is a universal function approximator, if allowed an unlimited number of neurons in its single hidden layer \citep{Cybenko:1989, Hornik.etal:1989}.  Since in theory a shallow network can do anything, what is the advantage of going deep?  For one thing, deeper architectures are capable of encoding certain types of functions far more efficiently than their shallow counterparts \citep{Montufar.etal:2014, Szymanski.etal:2014, Telgarsky:2015}.  The efficiency of function encoding is important for two reasons:
\begin{itemize}
\item  deep learning can tackle problems that may be computationally intractable with the shallow approach;
\item the fewer trainable parameters in a neural network, the lower the bound on their generalisation error \citep{Vapnik1998} due to decreased generic representational power \citep{Anthony.etal:2009, Bartlett.etal:2017},
\end{itemize}

While the first point is highly relevant for practical purposes, the latter is more interesting from the theoretical point of view.  An approximation that can be made to the same level of accuracy with significantly fewer parameters is likely to give better generalisation.  However, the notion of generalisation has little meaning when the function to be approximated is fully specified, as has been the case in theoretical comparisons of shallow versus deep architectures thus far.  Also, because of this presupposing of the desired mapping function, the existing proofs do not establish that deep representations are richer in general -- only that it is exceptionally efficient at certain types of approximations.  This does not exclude the possibility that there are times where shallow representations are better.    Although at the moment the empirical evidence suggests that going deeper does not hurt \citep{Zagoruyko.etal:2016}, we do not know that this is true in general.  

In this paper we examine the capabilities of different choices of neural network architecture from a different point of view.  Instead of contrasting the model complexity required for the same accuracy on a specified task, we compare the sizes of the hypothesis spaces from different variants of neural architecture of equivalent complexity (in terms of the total number of parameters).    Our analysis is based on counting the number of equivalence classes in the set of possible states for a neural network of a particular architecture where the equivalence relation corresponds to states that lead to the same function mapping.  We prove that the upper bound on the unique number of functions a neural network can produce is $O(\S/\Ufprod)$, where $\P$ is the total number of parameters, $\W$ is the cardinality of the set of values parameters can take, and $U_{\l}$ is the number of neurons in hidden layer $\l$.  This implies that given a fixed number of parameters, architecturally it is better to impart the computational complexity of the network into its depth rather than breadth in order to increase the model's function mapping capability.       

We also provide results of a numerical evaluation in small networks, which show that the actual number of unique function mappings, although much smaller than the theoretical bound and highly dependent on the choice of activation function, is nevertheless always larger in deeper architectures. 

\section{Neural network as a hypothesis space}

A neural network with a particular architecture is a hypothesis space, denoted as $\H$.  The architecture is specified through a set of hyperparameters.  Some of these, such as  the number of inputs $U_0=\M$, are dictated by the attributes of data the network needs to work with.  Other parameters, the number of hidden layers $\L$, number of hidden neurons $\U_{\l}$ in layers $\l=1,...,\L$, and the activation function $\f$ are chosen by the user.  Once the choice of the hyperparameters is made, the input-output mapping that the network provides will depend on the values of the weights and the biases on the connections between the neurons.  In this paper we will restrict ourselves to working with single-output networks.  The function produced by such network is:
\begin{equation}
h(\mathbf{x})=\sum_{j=1}^{\U_{\L}}w_{j}y^{[\L]}_{j}+w_{0},
\end{equation}
where $w_j$ and $w_0$ are respectively the weight and bias of the single output neuron, 
\begin{equation}
y^{[\l]}_{i}=\f\left (\sum_{j=1}^{\U_{(\l-1)}}w^{[\l]}_{ij}y^{[\l-1]}_{j}+w_{i0}^{[\l]}\right )
\end{equation}
is the output of the $i^{\mbox{th}}$ neuron in layer $l$, where $\sigma$ is some activity function, $w^{[\l]}_{ij}$ is a weight on the $j^{\mbox{th}}$ input from layer $l-1$, $w_{i0}^{[\l]}$ is the bias, and $U_0=\M$ with $y^{[0]}_j=x_j$ is $j^{\mbox{th}}$ attribute of input $\mathbf{x}\in\mathcal{R}^n$.

The total number of trainable parameters (weights + biases) in a fully connected single output feed-forward network is

\begin{equation}
\P=\sum_{\l=1}^{\L}(\U_{(\l-1)}+1)\U_{\l}+\U_\L,
\end{equation}
where, again, $\U_{0}=\M$.

A particular assignment of values to the weights and biases will be referred to as network's state.  The hypothesis space $\H$ given by a neural network of a particular architecture is the set of all possible functions that this architecture is capable of producing through all possible choices of its state.  Whenever there is a need to be explicit about the architecture, we will denote the corresponding hypothesis space $\H_{\M-\U_1-\hdots-\U_{\L}}$.

\section{Equivalence classes}\label{sec_uniquefs}

For a network of $\P$ parameters, where each can take on values from a finite set $\V$ of cardinality $\W=|\V|$, there is a total of $\S$ states.  However, different states can give rise to the same function mapping, and that is the equivalence relations we are interested in.  Identical function mappings despite different states is a consequence of the fact that the order of summation over neuron's weighted inputs does not matter with respect to its overall activity.  A subset of states with the same equivalence relation forms forms an equivalence class.  We want to establish how the choice of hyperparameters affects  the number of total number of equivalence classes within all of its states, and thus the number of unique function mapping, or the size of the hypothesis space, $|\mathcal{H}|$. 

\begin{figure}[t!]
\floatconts{fig_shalloworbits}
  {\caption{Three different permutations of 4 neurons in a hidden layer that do not affect the input/output mapping - the values of the input and output weights are preserved for a given neuron (labelled and coloured to aid visualisation).}}
  {
    \subfigure{
      \label{fig_shalloworbits_a}
      \includegraphics[width=0.28\columnwidth]{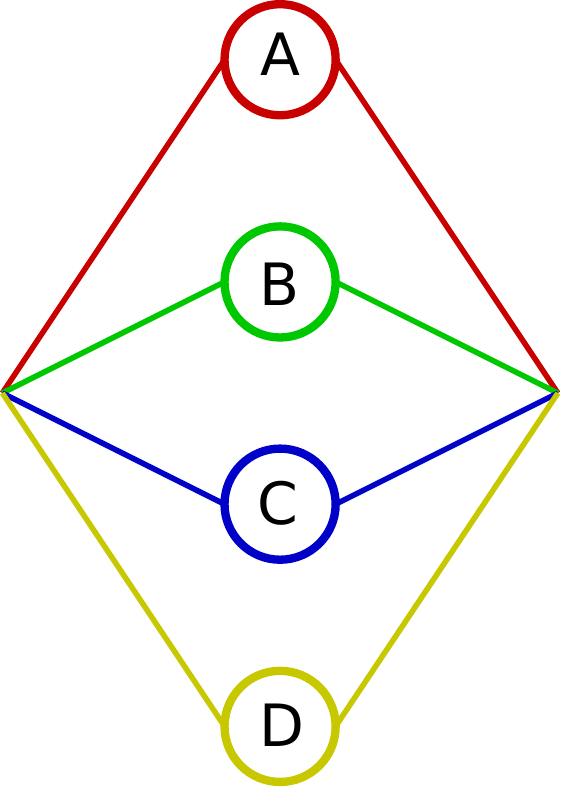}
    }\qquad
    \subfigure{
      \label{fig_shalloworbits_b}
      \includegraphics[width=0.28\columnwidth]{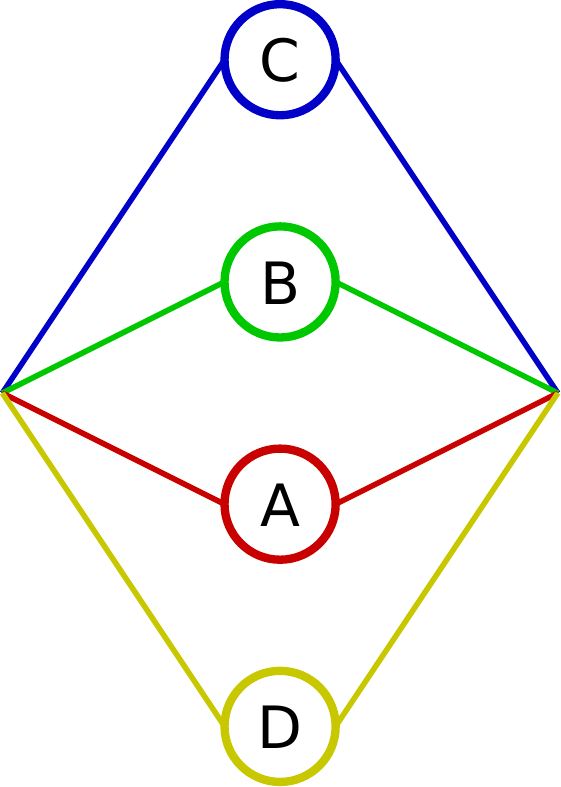}
    }\qquad
    \subfigure{
      \label{fig_shalloworbits_c}
      \includegraphics[width=0.28\columnwidth]{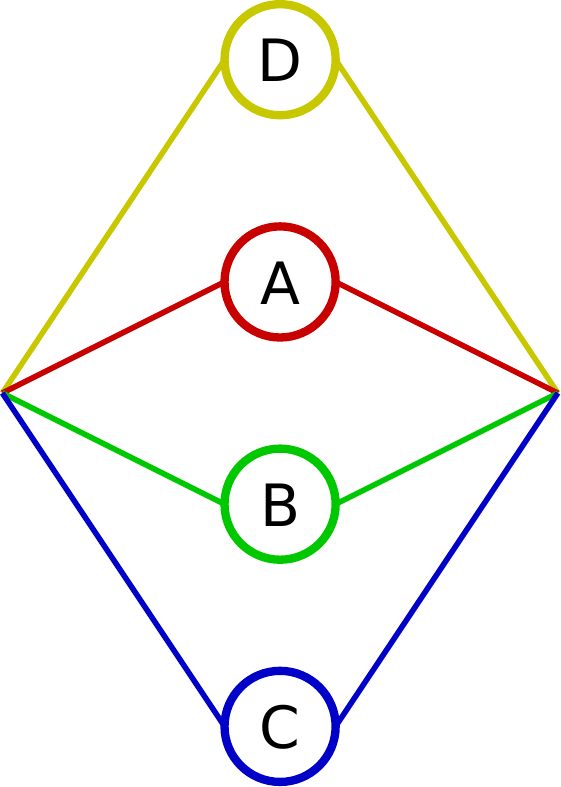}
    }\qquad
  }
\end{figure}

Let's examine a mapping from input to output of a single hidden layer as shown in \autoref{fig_shalloworbits} for an arbitrary choice of the weight values on the connections.  The change of state that does not affect the overall mapping is synonymous with a change in the positions of two (or more) neurons behaving as \emph{beads on a string}.  The neuron/bead can exchange its position with another neuron/bead, each taking along the \emph{strings} corresponding to its input and output connections.  The state of the network changes through a permutation of the weight values on the connections, but the overall computation does not.  As an example, the state change from  \autoref{fig_shalloworbits_a} to \autoref{fig_shalloworbits_b} is analogous to neuron A exchanging its position with neuron C.  \autoref{fig_shalloworbits_c} shifts the neurons with respect to \autoref{fig_shalloworbits_a} in such a way that A moves into position of B, B to C, C to D and D to A.  The neuron/bead analogy works for arbitrary number of inputs and outputs, thus also encompassing bias weights, which can be thought of as weights of a constant value input to all neurons in the layer.

Following the neuron/bead movement analogy it's fairly obvious that for a layer of $\U_\l$ neurons, and a particular choice of values on the connections, there are up to $\Uf$ permutations of the order of the summation producing the same mapping, regardless of the number of inputs and outputs of the layer.  There might be fewer than $\Uf$ permutations for certain choices of the values of the connections if the weights on neurons match in such a way that two (or more) neuron permutations produce identical state.  For instance, if all the input weights have exactly same value, and all the output weights have exactly same value, then all the neuron permutations produce exactly the same state.   

\begin{figure}[t!]
\floatconts{fig_deeporbits}
  {\caption{Four possible permutations of two consecutive layers, with two neurons per layer, that do not affect the input/output mapping - the values of the input and output weights are preserved for a given neuron (labelled and coloured to aid visualisation).}}
  {
    \subfigure{
      \label{fig_deeporbits_a}
      \includegraphics[width=0.46\columnwidth]{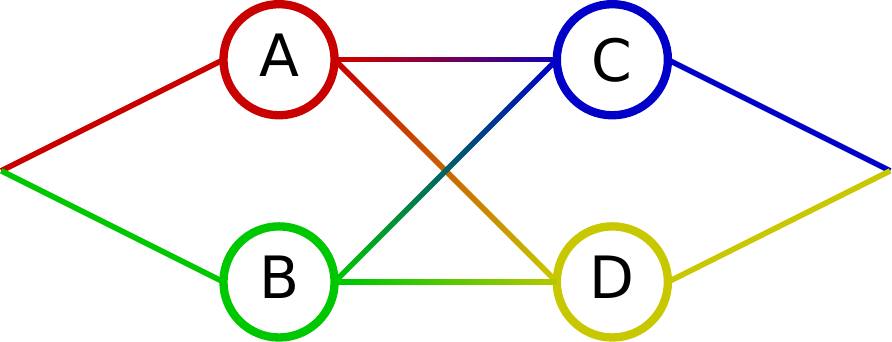}
    }\qquad
    \subfigure{
      \label{fig_deeporbits_b}
      \includegraphics[width=0.46\columnwidth]{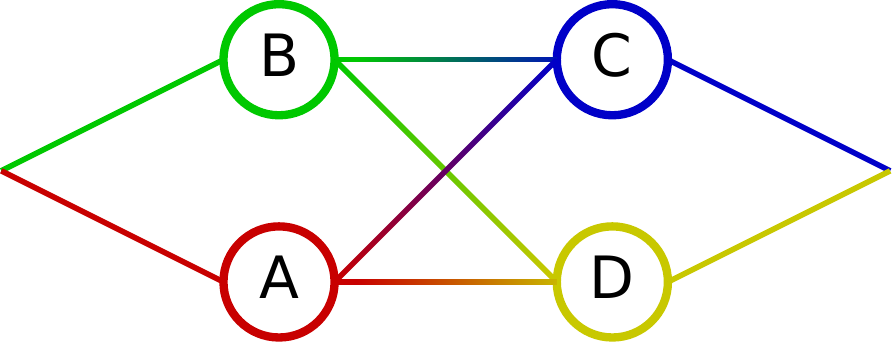}
    }\qquad
    \subfigure{
      \label{fig_deeporbits_c}
      \includegraphics[width=0.46\columnwidth]{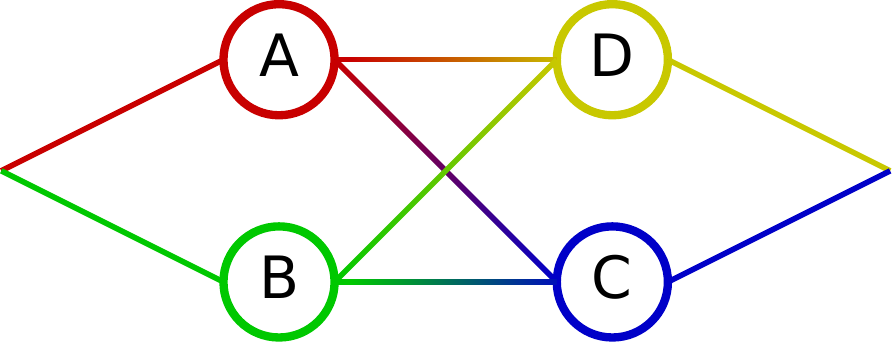}
    }\qquad
    \subfigure{
      \label{fig_deeporbits_d}
      \includegraphics[width=0.46\columnwidth]{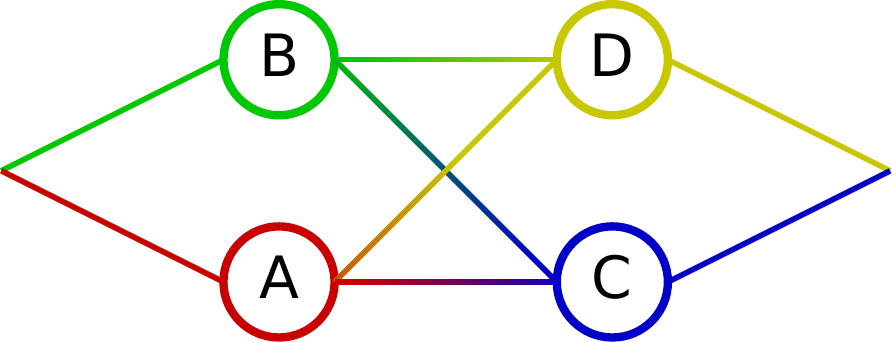}
    }\qquad
  }
\end{figure}

When accounting for the mapping capability of the combination of multiple layers, we need to account for all possible combinations of computation-preserving permutations of neurons of each layer.  \autoref{fig_deeporbits} shows all these combinations for a two hidden layer network with two neurons each.  Since each layer has two neurons, individually each gives rise to $2!$ equivalent permutations.  \autoref{fig_deeporbits_a} represents the first permutation of the neuron order in each layer, \autoref{fig_deeporbits_b} the second permutation of the first layer (from the left) along with the first permutation of the second layer, \autoref{fig_deeporbits_c} the first permutation of the first layer and second permutation of the second layer, and finally \autoref{fig_deeporbits_d} depicts the second permutation of the neuron order in both layers.  In general, depending on the choice of values of the parameters on the connections, there are up to $\S$ permutations of the neurons that preserve the function mapping of the network.  

If we take a finite set of $\W$ values, then there are $\S$ possible states for a network of an architecture with a total of $\P$ parameters.  If every state out of $\S$ was part of an equivalence class of at least $\Ufprod$ states producing the same function mapping, it would be trivially obvious that this network can give rise to no more than $\Hbound$ unique function mappings.  Situation is not that simple, since there are states (with same values on different parameters) that do not have $\Ufprod$ distinguishable permutations.  However, relying on fairly fundamental results from Group Theory \citep{Rotman:1994}, we can establish that indeed the upper bound on unique function mappings is $\Ufprod$.

\begin{restatable}[Unique Solutions]{theorem}{uniqueTheorem}
\label{thm_unique}%
The upper bound on the size of the hypothesis space $\H$ of a fully connected neural network with arbitrary activity function $\f$ is $O\left (\Hbound\right )$, where $\L$ is the number of hidden layers, $\U_{\l}$ is number of neurons in layer $\l$, and $\P$ is the total number of parameters and parameters $w_{ij}\in \V$, where $|\V| = \W$ is finite.  
\end{restatable}

\noindent The proof for the theorem, provided in \autoref{apx_unique_proof}, is based on application of Burnside's Lemma from Group Theory \citep{Rotman:1994}.

\subsection{Symbolic evaluation}

In order to get a sense of the tightness of the bound on $\Hsize{}$ given in \autoref{thm_unique}, we can run a symbolic evaluation over all possible states of network with $\P$ parameters chosen from a set of $\W$ symbols.  We can evaluate and compare the symbolic output from neural networks of different architectures for all $\S$ states and determine how many of these symbolic expressions are unique.  Though only possible for small $\W$ and $\P$, it still gives an idea on the tightness of bound on $\H$ for arbitrary $\f$.  

\begin{figure}[t!]
\includegraphics[width=\textwidth]{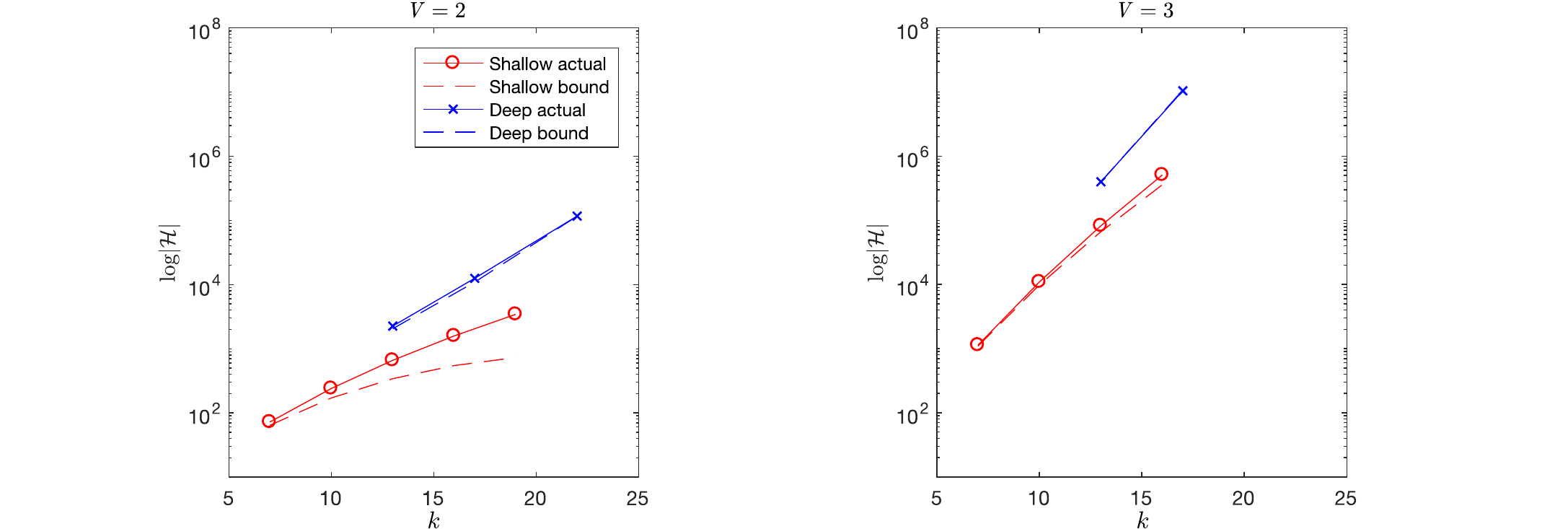}
\caption{Number of unique function mappings $\Hsize{}$ given by single input and single output neural networks against the total number of parameters $\P$; the red lines correspond to a single layer hidden network with the following architectures: $\H_{1-2}$ ($\P$=6), $\H_{1-3}$ ($\P$=9), $\H_{1-4}$ ($\P$=12), $\H_{1-5}$ ($\P$=15), $\H_{1-6}$ ($\P$=18); the blue lines correspond to a two hidden layer network with the following architectures: $\H_{1-2-2}$ ($\P$=12), $\H_{1-3-2}$ ($\P$=16), $\H_{1-3-3}$ ($\P$=21); there are two plots for $\W$=2 and $\W$=3}\label{fig_orbits_numerical}
\end{figure} 

\autoref{fig_orbits_numerical} shows the exact number (solid line) and the bound (dash line) of unique symbolic solutions for function mapping over unspecified function $\f$ plotted against  the number of parameters in a single-layer and two-hidden-layer neural network.  Note that the bound gets tighter as $\W$ increases.

\subsection{Numerical evaluation}\label{sec_numerical}

\begin{table}[htp]
\caption{$\Hsize{}$ and $\Ysize{}$ for two hypothesis spaces of $\P=12$ parameters for symbolic as well as other choices of activation functions.}
\label{tbl_numerical_fs}
\begin{center}
	\begin{tabular}{|c|l|r|r|r|r|}
	\hline
    & Network & & & & \\
    $W$ &  architecture & $\Hsize{}$ & $\Ysize{\mbox{ReLU}}$ & $\Ysize{\mbox{tanh}}$ & $\Ysize{\mbox{sigmoid}}$ \\
    \hline 
    \multirow{2}{*}{2} & $\H_{1-4}$ & 330 & 41 & 25 & 125 \\  
    & $\H_{1-2-2} $ & 1128 & 147 & 67 & 573\\
    \hline 
    \multirow{2}{*}{3} & $\H_{1-4}$ & 27405 & 277 & 321 & 2121 \\  
    & $\H_{1-2-2} $ & 132921 & 689 & 2165 & 16169\\
    \hline
	\end{tabular}
\end{center}
\end{table}    

To get a bit of an idea on the number of possible mappings in a practical scenario, we need a numerical evaluation over specific range of inputs and a choice of activation function.  We can evaluate all possible function mappings of hypothesis space $\H$ by considering model's output over a range of inputs for each $h\in\H$.  For single-input  single-output networks we evaluate $y_{i}=h(x_i)$ over $1001$ points from regularly samples range $x_i\in[-1,1]$ .  For each hypothesis $h\in\H$ we compute the output vector $\mathbf{y}=[y_1,...,y_{1001}]$.  Some of the hypotheses that give different functions symbolically might give identical mappings over the chosen range of input in the numerical evaluation.  Hence, we select the set of unique vectors (to within $1\times 10^{-4}$ Euclidean distance) to form a set of mappings $\mathbf{y}\in\Y_{\f}$, which corresponds to $\H$ for the choice of activation $\f$ over the selected range of input.   
\autoref{tbl_numerical_fs} shows the symbolically evaluated number of unique hypotheses against the number of unique vectors after numerical evaluation for different choice of activation functions.  The evaluated hypothesis spaces are $\H_{1-4}$ and $\H_{1-2-2}$, each with a total number of $\P=12$ parameters.  Numerical evaluation was done for $\W=2$ where $\V=\{-1,1\}$, and $W=3$ where $\V=\{1,0,1\}$.

It is hardly surprising that the choices for input range, allowed parameter values and activation function have a significant impact on the size of the corresponding hypothesis space.  The possibility of inputs and parameters of same value with opposite sign introduces additional symmetries in the internal computations of the network, thus reducing the number of unique function mappings.   ReLU introduces many extra symmetries, because it produces the same output for all negative activity.  So does tanh, because of its symmetry about 0.  Sigmoid gives rise to the richest hypothesis space.  

Note that, although for a given choice of $\f$ the number of unique functions is far below the upper bound given in \autoref{thm_unique}, the deeper/fewer neurons per layer hypothesis is always richer than the shallower/more neurons per layer version of the neural network.  

\section{Discussion}

We have show that upper bound on the size of the hypothesis space given by a neural network is dictated by the the number of neurons per layer.  For the same number of parameters deeper architecture (fewer neurons per more layers) gives a hypothesis space capable of producing more function mappings than a shallower one (with more neurons per fewer layers).


\bibliography{references}

\appendix
\begin{appendices}
\section{Proof of \autoref{thm_unique}}\label{apx_unique_proof}

The proof is a pretty straight forward application of Burnside's Lemma to count the number of equivalence classes of the states producing same function mapping in a neural network of particular architecture.  All the definitions and lemmas used here are proven in \citet{Rotman:1994}.

\uniqueTheorem*

\begin{proof}
Let's denote as $X$ the set of all possible states of a neural network of $\P$ parameters.  For our context, $|X|=\S$. To bound the size of $\H$, we can partition $X$ into equivalence classes of identical hypotheses and count the number of such classes. For the sake of completeness, we included some definitions.

\begin{definition}[Group; \cite{Rotman:1994}, pg. 12]\label{def_group}
A group is a nonempty set $G$ equipped with an associative operation $*$ containing an element $e$ such that:
\begin{enumerate}[label=(\roman*)]
\item $e * a = a = a * e$ for all $a\in G$
\item for every $a\in G$, there is an element $b\in G$ with $a * b=e = b *a $.
\end{enumerate}
\end{definition}

By \hyperref[def_group]{Definition \ref*{def_group}} the set of bijections $X\times X$ (or permutations) of the $\P$ parameters that do not affect the overall function mapping of the network is a group.  The operation $*$ is a permutation.  Indeed, we can apply a permutation to a permutation and obtain another permutation.  The identity permutation $e$ is a permutation that maps every element onto itself.  Following the explanations from \autoref{sec_uniquefs} we can see that group $G$ consist of $\Ufprod$ parameter permutations   isomorphic to the product of the permutations of the order of $\U_{\l}$ neurons in each hidden layer $\l=1,...,\L$.   

\begin{definition}[G-set; \cite{Rotman:1994}, pg. 55]
If X is a set and G is a group, then X is a G-set if there is a function $\alpha: G\times X\mapsto X$ (called an action), denoted by $\alpha: (g,x)\mapsto gx$, such that:
\begin{enumerate}[label=(\roman*)]
\item $e * x = x$ for all $x\in X$; and
\item $g(hx)=(gh)x$ for all $g,h \in G$ and $x\in X$. 
\end{enumerate}
\end{definition}

$X$ is a G-set, because permutations from $G$ re-order the values of parameters of the network creating another state in $X$.  The action is the re-ordering of the parameter values dictated by the permutation $g\in G$.  Condition $(i)$ is satisfied by the identity permutation, which will map network state $x$ to itself.  Condition $(ii)$ is satisfied by the fact that application of several permutations is associative. 

\begin{definition}[G-orbit; \cite{Rotman:1994}, pg. 56] 
If $X$ is a G-set and $x\in X$, then the \textbf{G-orbit} of $x$ is:
\begin{equation*}
\mathcal{O}(x)=\{gx: g\in G\}\subset X
\end{equation*}  
\end{definition}

The G-orbits we are interested in are the subsets of $X$ created by application of all neuron swapping permutations $g\in G$ to all states $x\in X$.  These subsets partition $X$, each containing the states that produce the same hypothesis.  We need to determine how many G-orbits there are in $X$.

\begin{lemma}[Burnside's Lemma; \cite{Rotman:1994}, pg. 58]\label{lmm_burnside}
If X is a finite $G$-set and $N$ is the number of $G$-orbits of $X$, then
\begin{equation*}
N=(1/|G|)\sum_{\tau\in G}F(\tau),  
\end{equation*}
where, for $\tau\in G$, $F(\tau)$ is the number of $x\in X$ fixed by $\tau$.
\end{lemma}

We have established that when $\W$ is finite, the set of network states $X$ is a finite set, and it is a G-set acted on by permutations of network parameters resulting from changing the order of summation of neuron output in network layers, where $|G|=\prod_l U_l!$.  $N$ is the number of G-orbits in $X$ created by actions of permutations from $G$, and thus it's the number of unique function mappings that a neural network can produce.  The last thing we need to evaluate in order to get $N$ is $F(\tau)$.  

In our context $F(\tau)$ specifies how many unique states a permutation $\tau\in G$ of $\P$ elements can create when all possible choices of $w_{ij}$ for the $\P$ elements are considered.  The answer is given by the following lemma found in \citet{Rotman:1994} (we changed the notation and analogy from colours to parameter values)

\begin{lemma}[\cite{Rotman:1994}, pg. 60]\label{lmm_colours}
Let $\V$ be a set with $|\V|=\W$, and let $G$ be a subset of all possible permutation of $\P$ elements.  If $\tau\in G$, then $F(\tau)=\W^{t(\tau)}$, where $t(\tau)$ is the number of cycles occurring in the complete factorisation of $\tau$.
\end{lemma}

Every permutation can be expressed as a factor of disjoint cycles.  For example, a permutation written as $(1,2)(3,4,5)(6)(7)$ denotes the following reordering of seven elements in $4$ cycles:
\begin{itemize}
\item element $2$ swaps with element $1$;
\item element $3$ goes into place of element $4$, which in turns goes into place of element $5$, which goes into place of element $3$;
\item element $6$ is fixed, its position remains unchanged,
\item element $7$ is fixed.
\end{itemize}

Since by \hyperref[lmm_colours]{Lemma \ref*{lmm_colours}} $F(\tau)=\W^{t(\tau)}$, where $t(\tau)$ is the number of cycles, the sum in \autoref{lmm_burnside} will be dominated by the permutation $\tau\in G$ with the largest number of cycles.  For a permutation of $\P$ elements, the largest possible number of cycles is $t(\tau)=\P$, and it's given by the identity permutation, $\tau=e$.  Hence, as $W$ increases, we have 
\begin{equation*}
N=O\left(\W^{t(e)}/|G|)\right)=O\left (\S/\prod_{\l}\Uf\right )
\end{equation*} 

Given that the set $X$ has $N=O\left (\Hbound\right )$ G-orbits with respect to all combinations of neuron-swapping permutations in all individual neural networks, we have an upper bound on the number of functions a neural network of a particular architecture can generate.  Thus $\Hsize{}\le O\left (\Hbound\right )$.
\end{proof}
  
The tightness of the bound $\Hsize{}\le O\left (\Hbound\right )$ depends on the choice of activation function $\f$ and the set of parameter values $\V$.  During numerical evaluation, as shown in \autoref{sec_numerical}, extra symmetries might arise inside the neural network, which can result in different G-orbits in $X$ producing the same function mapping.  

\end{appendices}

\end{document}